\documentclass[sigconf]{acmart}
\AtBeginDocument{%
  }

\setcopyright{acmlicensed}
\copyrightyear{2018}
\acmYear{2018}
\acmDOI{XXXXXXX.XXXXXXX}
\acmConference[Conference acronym 'XX]{Make sure to enter the correct
  conference title from your rights confirmation email}{June 03--05,
  2018}{Woodstock, NY}
\acmISBN{978-1-4503-XXXX-X/2018/06}

\usepackage{pifont}   

\usepackage{scalefnt}
\usepackage{wrapfig}
\usepackage{tabularx} 
\usepackage{colortbl}
\usepackage{xcolor,bm}
\usepackage{bm}
\usepackage{enumitem}

\usepackage{amsmath}
\usepackage{xspace}

\providecommand{\Ntr}{\mathcal{N}_{\mathrm{tr}}}
\providecommand{\Nval}{\mathcal{N}_{\mathrm{val}}}
\providecommand{\Ncal}{\mathcal{N}_{\mathrm{cal}}}

\DeclareMathOperator*{\argmin}{arg\,min}
\providecommand{\norm}[1]{\ensuremath{\left\lVert #1 \right\rVert}}
\providecommand{\sg}{\operatorname{sg}}
\providecommand{\ind}{\mathbf{1}}

\definecolor{orbitblue}{HTML}{E4EFF8}

\begin{document}

\newcommand{\method}{{\fontfamily{lmtt}\selectfont \textbf{CLOSURE}}\xspace}

\title{CLOSURE: Cross-Channel Closure for Multivariate Time Series Anomaly Detection}


\author{Fan Xu}
\affiliation{%
  \institution{University of Science and Technology of China}
  \city{Hefei}
  \state{Anhui}
  \country{China}
}
\email{markxu@mail.ustc.edu.cn}

\author{Hao Wu}
\affiliation{%
  \institution{Tsinghua University}
  \city{Beijing}
  \country{China}
}
\email{wu-h25@mails.tsinghua.edu.cn}

\author{Zhi Song}
\affiliation{%
  \institution{City University of Hong Kong}
  \city{Hongkong}
  \country{China}
}
\email{songzhics@gmail.co}

\author{Hanbo Huang}
\affiliation{%
  \institution{Shanghai Jiao Tong University}
  \city{Shanghai}
  \country{China}
}
\email{hhuang417@sjtu.edu.cn}

\author{Wenzhuo Ma}
\affiliation{%
  \institution{Tianjin University}
  \city{Tianjin}
  \country{China}
}
\email{3353147727j@gmail.com}

\author{Nan Wang}
\affiliation{%
  \institution{Beijing Jiaotong University}
  \city{Beijing}
  \country{China}
}
\email{wangnanbjtu@bjtu.edu.cn}

\author{Wei Gong}
\authornotemark[2]
\affiliation{%
  \institution{University of Science and Technology of China}
  \city{Hefei}
  \state{Anhui}
  \country{China}
}
\email{weigong@ustc.edu.cn}

\author{Xibin Zhao}
\authornotemark[2]
\affiliation{%
  \institution{Tsinghua University}
  \city{Beijing}
  \country{China}
}
\email{zxb@tsinghua.edu.cn}

\renewcommand{\shortauthors}{Fan Xu et al.}

\renewcommand{\shortauthors}{Trovato et al.}

\begin{abstract}
Multivariate time series produced by industrial installations and server clusters are under continuous monitoring, and anomaly detection is the first line of defense before a local fault escalates into a system-wide failure. In the absence of labels, the dominant paradigm trains a reconstruction model on normal data and scores the residual to reach a verdict. Recent work has steadily enriched the expressiveness of the recovery map through Transformers, diffusion models, and cross-scale associations. A residual, however, only certifies that an observation can be recovered; it does not certify that the variables still explain one another as they do under normal operation. A coordinated yet unfamiliar operating transition inflates the residual without breaking any relationship, whereas a sensor drifting away from its peers may leave the residual silent. Methods that fold inter-variable dependencies into the reconstruction objective inherit this ambiguity, because their relational evidence is ultimately supervised by the same recovery loss. We therefore propose \method{}, which evaluates recoverability and cross-channel consistency as two normal-data prediction tasks learned in separation. The recoverability task is solved by a diffusion reconstructor conditioned on a discrete memory of normal temporal bases. The consistency task is solved by a relation-closure operator that predicts every channel exclusively from the others, where the target's own trajectory is withheld and sparse signed gates decide which peers are actually used. The two scores are reconciled only through calibration on normal data. Extensive experiments show that \method{} consistently outperforms recent reconstruction-, frequency-, and topology-based detectors on public benchmark streams, confirming the value of examining consistency alongside recoverability.

\end{abstract}




\maketitle

\section{Introduction}
\label{sec:introduction}

Multivariate time series produced by industrial installations and server clusters are under continuous monitoring, and anomaly detection is the first line of defense before a local fault escalates into a system-wide failure \citep{tuli2022tranad,su2019robust}. Labeled failures are scarce and never enumerate the ways a system can break, so detectors are typically fitted to normal observations alone. Reconstruction supplies the canonical objective for this setting: learn to reproduce normal windows, then read the residual as evidence of departure. Recurrent latent models and adversarial autoencoders established the paradigm \citep{su2019robust,zhou2019beatgan}; Transformers widened the temporal context \citep{tuli2022tranad}; diffusion models recover observations through denoising and imputation \citep{zhang2022grelen,chen2023imdiffusion,xiao2023imputation,physdiff2025}; and cross-scale reconstruction with shared cross-window context further enriches what a model can recover \citep{crossad2025}. These advances strengthen the representation of normal behavior. Whether they strengthen detection depends on a separate question: how faithfully reconstruction difficulty tracks abnormality.

\begin{figure*}[t]
    \centering
    \includegraphics[width=\textwidth]{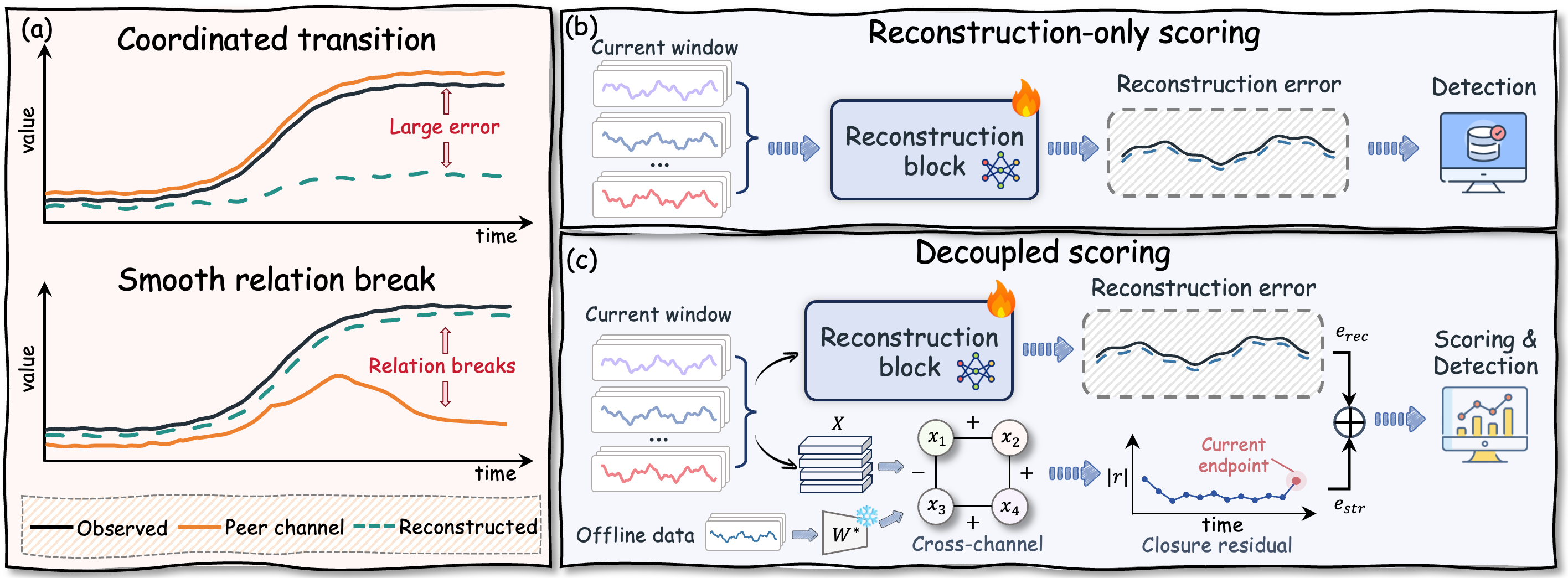}
    \vspace{-10pt}
    \caption{\textbf{Reconstructable does not necessarily mean consistent.} (a) A coordinated transition can be hard to reconstruct while every relationship stays intact, whereas a smooth relation break can remain easy to reconstruct. (b) Reconstruction-only scoring compares an observation with its recovery. (c) \method additionally compares the same observation with predictions issued by its peer channels. The depicted $W^\star$ denotes frozen connectivity.}
    \vspace{-5pt}
    \label{fig:overview}
\end{figure*}

Reconstruction difficulty and abnormality need not coincide. Consider a healthy operating transition in which several sensors change together while preserving their usual coordination. The trajectory may be unfamiliar to the reconstructor and yield a large residual even though no channel relationship has been violated. The converse is equally plausible. A drifting sensor may follow a smooth curve that the model reproduces accurately, while it quietly stops tracking the channels with which it normally varies. Figure~\ref{fig:overview} contrasts these possibilities. They separate two properties of a single window, \emph{recoverability} and \emph{cross-channel consistency}, and expose an ambiguity in the residual: agreement with a learned reconstructor does not by itself establish agreement with normal relations. 

Dependencies already play an important role in anomaly detection. Learned sensor graphs and inter-metric embeddings model interactions among variables \citep{deng2021graph,li2021multivariate, xu2024revisiting}. CATCH learns channel interactions inside frequency patches \citep{wu2024catch}, and ScatterAD couples temporal and topological views \citep{scatterad2025}. A related line uses temporal association discrepancies as detection evidence \citep{xu2021anomaly,yang2023dcdetector}. These methods establish the value of relational evidence. Our question is therefore not whether dependencies matter, but where their supervision comes from. When inter-variable dependencies are folded into the reconstruction objective, the relational evidence is ultimately shaped by the same recovery loss, and the resulting score inherits the ambiguity above. Such a reference also needs a criterion of its own for deciding which relations enter it. A strong association within one normal interval may reflect a temporary operating condition, and a merely strong correlation does not identify the peers that a nonlinear predictor should actually use. Turning these associations directly into prediction weights therefore grants uncertain relations an influence on detection before their usefulness has been measured. We instead ask how to construct a cross-channel reference whose supervision, and whose selection of relations, are both established by a prediction task separate from reconstruction.

According to these, we introduce \method, a detector that evaluates recoverability and cross-channel consistency as two normal-data prediction tasks learned in separation. The consistency task is carried by a \emph{relation-closure operator}. Each target receives nonlinear messages from its peers through learned signed coefficients and sparse gates, and never from itself; the disagreement between the observed channels and their peer predictions forms a \emph{closure residual} \citep{batson2019noise2self}. Because a sustained departure can raise a window's residual background together with its endpoint, we score the endpoint relative to that background, which makes the criterion sensitive to a localized change in coordination. The recoverability task is carried by a diffusion reconstructor conditioned on a discrete memory of normal temporal bases \citep{physdiff2025}. The relation operator is neither an input to this reconstructor nor a target of its training loss, and the two paths meet only after their scores are calibrated on normal data. This separation lets us hold the reconstructor fixed and vary the relational criterion alone, so that a change in detection is attributable to relation learning rather than to improved recovery.
In summary, we make the following key contributions:

\vspace{-5pt}
\begin{itemize}[leftmargin=2em,itemsep=-0.1em]
\item[\ding{182}] \textbf{Decoupled scoring.} We formulate a detector that evaluates temporal recoverability alongside cross-channel consistency, keeping the relational reference outside the reconstruction objective and combining the two paths only at the decision stage.

\item[\ding{183}] \textbf{Learned sparse closure.} We instantiate the relational criterion with signed coefficients, per-edge gates, and source-only nonlinear messages, and show that row normalization erases absolute relation strength, motivating a separate gate parameter for selection.

\item[\ding{184}] \textbf{Controlled evaluation.} We prove that a target-only perturbation enters the closure residual unattenuated, and experiments with a frozen reconstructor show that learned relational evidence improves affiliation F1 beyond reconstruction alone.
\end{itemize}

\section{Related Work}
\label{sec:related}

\subsection{Reconstruction-Based Detection} 
Reconstruction-based detectors learn normal dynamics and score deviations through reconstruction error or likelihood. Recurrent models, adversarial autoencoders, and Transformers provide different temporal representations for this task \citep{su2019robust,zhou2019beatgan,tuli2022tranad, xu2025advanced}. Diffusion approaches formulate recovery as conditional imputation \citep{chen2023imdiffusion,xiao2023imputation}. PhysDiff \citep{physdiff2025} conditions corruption and attention on decomposed temporal bases, while CrossAD \citep{crossad2025} reconstructs fine-scale observations from coarser views and shares context across windows. These designs enrich what a model can recover the residual. however, they still report agreement with a single recovery map.

\subsection{Dependency Modeling} 
Dependencies can guide signal modeling or become detection targets. GDN \citep{deng2021graph} and InterFusion \citep{li2021multivariate} encode channel interactions through sensor graphs and hierarchical representations, respectively. CATCH \citep{wu2024catch} learns frequency-dependent channel masks, whereas MSCRED \citep{zhang2019mscred} reconstructs multiscale inter-sensor signature matrices and scores their residuals. Another line examines temporal associations through association discrepancy and dual-attention contrastive learning \citep{xu2021anomaly,audibert2020usad,park2020learning,yang2023dcdetector}. ScatterAD \citep{scatterad2025} models temporal-node scattering and aligns temporal and topological representations. Across these designs, the relational evidence is produced within the reconstruction or representation objective, so its reference and the recovery map are shaped by the same loss.

\subsection{Normality Constraints and Scoring} 
Normal references and constraints on the recovery map address the risk of accurately reconstructing anomalies. MemAE \citep{gong2019memorizing} retrieves normal latent prototypes for decoding, and vector quantization provides a discrete representation for such prototypes \citep{van2017neural}. MEMTO \citep{song2023memto} extends memory-guided reconstruction with latent-memory deviations that weight reconstruction errors, demonstrating that a normal reference can also supply an additional scoring criterion. IGAD \citep{sun2025idempotent} instead imposes idempotency and tightness losses on the reconstruction map. These constraints act on the recovery map or on the representation it decodes from, and their criterion is still read from how well the observation is reproduced. 

These methods use learned dynamics, dependencies, or normal prototypes to define detection evidence. \method instead complements reconstruction with a relational reference that is learned separately and never enters the reconstruction objective. Its signed closure operator evaluates whether observed channels remain mutually explainable when each target is withheld from its own prediction, rather than reconstructing a relation matrix or measuring distance to a latent prototype.
\section{Methodology}
\label{sec:method}

\textbf{Problem setup.}
Let $x_t\in\mathbb R^D$ contain the retained channels of a multivariate series and $X_t=[x_{t-L+1},\ldots,x_t]^\top\in\mathbb R^{L\times D}$ its length-$L$ window. We seek a score $a_t$ for the \emph{observed endpoint}, using normal data for fitting and calibration. The normal sequence is divided chronologically into training $\Ntr$, validation $\Nval$, and calibration $\Ncal$. Windows stay within their assigned segments, with $L-1$ unused timestamps between segments. All inputs use training-fitted scaling.


\textbf{Overview.}
Figure~\ref{fig:method} follows one window through two paths that answer different questions. \ding{182} \textit{\textbf{Temporal reconstruction}} measures how accurately a learned model recovers $X_t$ given normal temporal context. \ding{183} \textit{\textbf{Relational closure}} measures whether the channels of $X_t$ remain mutually explainable under a learned normal relation reference. Three components are fitted separately: a discrete basis codec, the reconstructor, and the relation predictor. The relation path receives neither memory codes nor reconstruction outputs, and contributes no gradient to them. Both paths are frozen at detection and meet only through normal-calibrated scoring.

\begin{figure*}[t]
    \centering
    \includegraphics[width=1.0\textwidth]{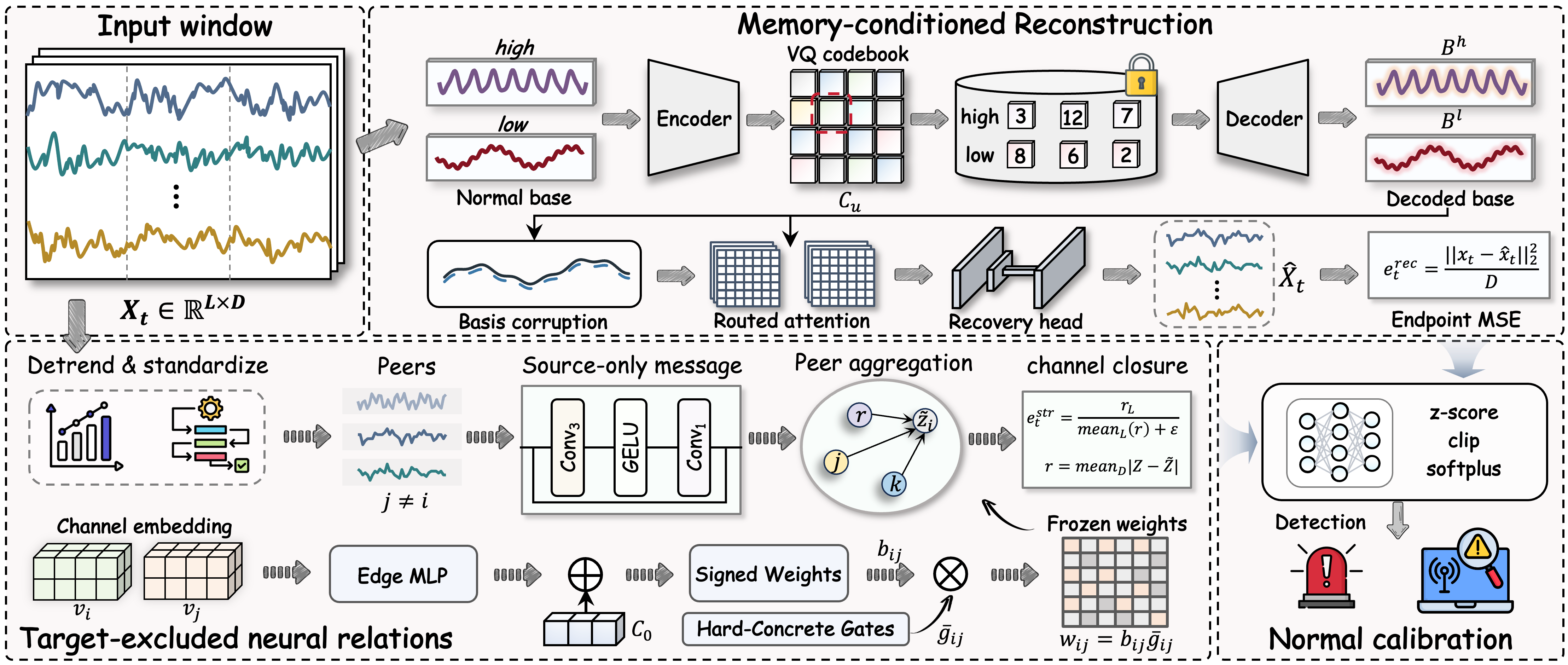}
    \vspace{-5pt}
    \caption{\textbf{Overview of \method.} Paired codebooks store normal temporal bases as code indices; decoded bases condition both corruption and attention. Learned signed coefficients, per-edge gates, and source-only messages predict each target from its peers, never from itself. Independently calibrated scores are combined for detection.}
    \label{fig:method}
\end{figure*}

\subsection{Memory-Conditioned Temporal Reconstruction}
\label{sec:recoverability}

The reconstruction path measures how accurately an entire window can be recovered from normal temporal context, and reports the resulting discrepancy at the observed endpoint. We introduce the basis-conditioned diffusion backbone and represent the normal bases with learned discrete codecs, so that recovery is conditioned on stored normal patterns rather than on the test window itself.

\noindent\textbf{A discrete normal basis memory.}
High- and low-frequency bases are extracted from normal training windows by adaptive multicomponent decomposition. For band $u\in\{h,l\}$, channel-wise bases are split into length-$p$ patches, normalized by training-fitted scales, and quantized against a codebook $\mathcal C_u=\{c^u_k\}_{k=1}^{K}$ through a straight-through estimator~\citep{van2017neural}:
\begin{equation}
\begin{gathered}
z^u=E_u(b),
\quad
q_u(b)=\argmin_{1\le k\le K}\norm{z^u-c^u_k}_2^2,\\
\widehat b=D_u\bigl(z^u+\sg[c^u_{q_u(b)}-z^u]\bigr).
\end{gathered}
\label{eq:memory-quantization}
\end{equation}
Each band owns a codec shared across channels. A memory slot stores the paired code-index sequences of one normal template; reading a slot decodes both bands and restores channel scales to give $\mathcal B=(B^h,B^l)$ with $B^u\in\mathbb R^{L\times D}$. The shared slot index preserves band pairing. Slots are \emph{sampled} from stored normal context, never retrieved by matching a test query, and no write occurs during detection.

\noindent\textbf{Basis-conditioned recovery.}
At a fixed diffusion index $s$ with $\rho_s=\sqrt{1-\bar\alpha_s}$, the decoded bases enter both corruption and attention. Writing $Q,K,V\in\mathbb R^{L\times d}$ for token projections and $U^u\in\mathbb R^{L\times d}$ for basis keys,
\begin{equation}
\begin{gathered}
X_t^{(s)}=\sqrt{\bar\alpha_s}\,(X_t+\Delta_n)
+\rho_s\bigl[\eta\varepsilon+(1-\eta)(\lambda_s^hB^h+\lambda_s^lB^l)\bigr],\\
A=\operatorname{softmax}\!\Bigl(\frac{QK^\top}{\sqrt d}
+\!\!\sum_{u\in\{h,l\}}\!\!\bigl[G^u\odot S^u+\widetilde G^u\odot\widetilde S^u\bigr]\Bigr),
\\
S^u=\frac{Q(U^u)^\top}{\sqrt d},
~~
\widetilde S^u=\frac{\widetilde Q(\widetilde U^u)^\top}{\sqrt d},
\end{gathered}
\label{eq:basis-recovery}
\end{equation}
where $\varepsilon\sim\mathcal N(0,I)$, $\lambda_s^u$ weights each band, and $\Delta_n$ is a training-only channel offset repeated across time. Corruption uses the full multichannel bases; attention uses their channel-averaged projections. 
Since attention aggregates $AV$, the values still come from the observation. Two attention--feed-forward blocks and a token-wise MLP yield the recovery and its endpoint score:
\begin{equation}
\widehat X_t=g_{\mathrm{rec}}(H_t)-\Delta_n,
\quad
e_t^{\mathrm{rec}}=\frac{1}{D}\norm{x_t-\widehat x_t}_2^2 .
\label{eq:reconstruction-score}
\end{equation}
Recovery is a forward pass at one fixed index, not an iterative reverse trajectory. At detection $\Delta_n=0$ and the codec is frozen, while basis sampling and Gaussian corruption follow the evaluation seed.

\subsection{Target-Excluded Relational Closure}
\label{sec:closure-operator}

The residual of Eq.~\ref{eq:reconstruction-score} leaves undetermined whether the observed channels still explain one another, which is the assessment this path supplies. Let $Z_t\in\mathbb R^{L\times D}$ denote $S(X_t)$ standardized channel-wise across the window, and $z_j=Z_{:,j}$. A predictor able to read its target can copy that trajectory and report agreement without consulting any peer, so every message entering target $i$ comes from sources $j\ne i$ alone. Selection cannot be delegated to the coefficients either.

\begin{lemma}[\textit{Normalization erases absolute strength}]
Let $\mathcal W_\epsilon(C)_{ij}=C_{ij}/(\sum_k|C_{ik}|+\epsilon)$ denote row-wise $L_1$ normalization.\label{lem:scale-erasure}
For any $\alpha>0$ and any matrix $E$ with $s_i(E)=\sum_k|E_{ik}|>0$,
\begin{equation}
\norm{\mathcal W_\epsilon(\alpha E)_{i:}-\mathcal W_0(E)_{i:}}_1
=\frac{\epsilon}{\alpha\,s_i(E)+\epsilon}
\;\xrightarrow[\ \alpha s_i(E)/\epsilon\to\infty\ ]{}\;0 .
\label{eq:scale-erasure}
\end{equation}
\end{lemma}

\noindent
A row of uniformly weak associations therefore still receives nearly unit total weight, so we separate \emph{signed contributions}, stating how a peer explains its target, from \emph{gates}, stating whether it participates.

\noindent\textbf{Signed peer contributions.}
Let $C(\cdot)$ return a temporal Pearson correlation matrix with zero diagonal. Normal training windows $\mathcal W_{\mathrm{tr}}$ supply a signed reference over all off-diagonal candidates, which a learned correction then modifies:
\begin{equation}
\begin{aligned}
C_0&=\frac{1}{|\mathcal W_{\mathrm{tr}}|}\!\!\sum_{X\in\mathcal W_{\mathrm{tr}}}\!\!C(S(X)),\\
\omega_{ij}=C_{0,ij}+&a_\phi([v_i;v_j]), ~~
b_{ij}=\frac{\omega_{ij}}{\sum_{k\ne i}|\omega_{ik}|+\epsilon},
\end{aligned}
\label{eq:signed-coefficients}
\end{equation}
where $v_i\in\mathbb R^r$ is a channel embedding and $a_\phi$ an edge MLP. Every pair stays a candidate, including pairs with zero initial correlation, since the additive correction can reverse a sign or introduce an absent connection. Unlike $C_0$, the learned coefficients are asymmetric, as predicting $i$ from $j$ differs from the reverse.

\noindent\textbf{Per-edge gates.}
Each candidate carries a trainable logit $\alpha_{ij}$, relaxed into a hard-concrete gate~\citep{louizos2017learning} so that admission is learned by gradient descent while exact zeros stay reachable:
\begin{equation}
\begin{aligned}
g_{ij}(u)&=\operatorname{clip}_{[0,1]}\!\Bigl(\gamma+(\zeta-\gamma)\,
\sigma\bigl(T_g^{-1}[\log\tfrac{u_{ij}}{1-u_{ij}}+\alpha_{ij}]\bigr)\Bigr), \\
w_{ij}&=b_{ij}\,g_{ij}(u),
\end{aligned}
\label{eq:sparse-gates}
\end{equation}
with $u_{ij}\sim\mathcal U(0,1)$, replaced by its deterministic counterpart at validation and detection. Normalization precedes gating and none follows, since renormalizing would cancel exactly the attenuation the gates express; by Lemma~\ref{lem:scale-erasure} this ordering keeps admission decidable.

\noindent\textbf{Source-only nonlinear messages.}
Each source trajectory therefore passes through a residual temporal network conditioned on the pair identity:
\begin{equation}
\psi_\phi(z_j;v_i,v_j)
=z_j+\operatorname{Conv}_1\!\left(\operatorname{GELU}\!\left(
\operatorname{Conv}_3(z_j)+E_\phi[v_i;v_j]\right)\right).
\label{eq:neural-closure}
\end{equation}
A single aggregation step forms the prediction, $\bigl[G_{\phi,\alpha}(Z;u)\bigr]_{:,i} = \sum\nolimits_{j\ne i} w_{ij}\,\psi_\phi(z_j;v_i,v_j)$. Subscripts denote kernel widths. The offset $E_\phi[v_i;v_j]$ encodes channel identities, not the target's trajectory, and the single hop admits no all-channel mixer that could route target information back through a peer.

\noindent\textbf{Closure as the relational score.}
With $G^\star$ the frozen predictor and $\widetilde Z_t=G^\star(Z_t)$, the \emph{closure residual} measures how far the observation lies from its peer predictions, aggregated over all retained channels and compared against its own window background:
\begin{equation}
r_{t,\ell}=\frac{1}{D}\sum_{i=1}^{D}
\bigl|Z_t(\ell,i)-\widetilde Z_t(\ell,i)\bigr|,
\quad
e_t^{\mathrm{str}}=\frac{r_{t,L}}{L^{-1}\sum_{\ell=1}^{L}r_{t,\ell}+\epsilon}.
\label{eq:closure-score}
\end{equation}
The scored channel set is fixed independently of the gates, so closing a row's connections changes the predictor without removing that target from evaluation. The ratio isolates a \emph{localized change in coordination}; a departure sustained across the window raises numerator and background alike, so this score does not measure anomaly duration.

\subsection{Training and Detection}
\label{sec:training}

\textbf{Separate prediction objectives.}
Normal windows supervise both paths, but the two objectives update disjoint parameters:
\begin{equation}
\begin{aligned}
\mathcal L_{\mathrm{rec}}(\theta)
&=\mathbb E_{X,\varepsilon,\mathcal B,\Delta_n}
\left[\frac{\norm{X-f_\theta(X;\varepsilon,\mathcal B,\Delta_n)}_F^2}{LD}\right],\\
\mathcal L_{\mathrm{rel}}(\phi,\alpha)
&=\mathbb E_{Z,u}\left[\frac{1}{LD}
\sum_{\ell=1}^{L}\sum_{i=1}^{D}
\operatorname{Huber}_1\!\left(Z(\ell,i)-[G_{\phi,\alpha}(Z;u)]_{\ell,i}\right)\right] \\
&+\lambda\mathcal S(\alpha),
\end{aligned}
\label{eq:training-objective}
\end{equation}
where $X$ is drawn from $\Ntr$, $Z$ is its standardized detrended counterpart, and $\mathcal S(\alpha)$ penalizes the mean probability that a gate stays open. The codec is frozen before the reconstructor is trained, and $\mathcal L_{\mathrm{rel}}$ sends no gradient to $\theta$ or the codec. Every target contributes to the Huber loss even when all of its incoming gates close, so pruning cannot lower the objective by dropping a difficult target. Normal validation selects both checkpoints, by reconstruction error for $\theta$ and by deterministic unpenalized Huber loss for $(\phi,\alpha)$.

\noindent\textbf{Normal-calibrated decision.}
The two errors share no scale. Each is standardized by its own branch statistics on $\Ncal$ and passed through softplus to give nonnegative evidence, and the fused score is their equal-weight mean:
\begin{equation}
a_t=\tfrac12\left(u_t^{\mathrm{rec}}+u_t^{\mathrm{str}}\right),
\quad
\widehat y_t=\ind\bigl[a_t>Q_{\delta}(\{a_s\}_{s\in\mathcal I_{\mathrm{cal}}})\bigr],
\label{eq:calibration-decision}
\end{equation}
where $\delta$ is the $99$th percentile of $a_t$ over normal calibration endpoints; fusing before thresholding is not equivalent to thresholding each branch and taking a disjunction. Parameters, score transforms, and $h$ stay fixed at test time, and labels enter evaluation only, without point adjustment. This quantile is an operating point, not a guarantee of false-alarm control under distribution shift.

\subsection{Theoretical Analysis}
\label{sec:theory}

Both statements concern the frozen operator $G^\star$ in fixed standardized coordinates, before the endpoint ratio of Eq.~\ref{eq:closure-score}. The first identifies what the architectural restriction guarantees, the second which perturbations the residual can expose.

\begin{theorem}[\textit{Target-excluded closure response}]\label{thm:target-exclusion}
Let $E(Z)=Z-G^\star(Z)$ and let $e_i$ be the $i$th coordinate vector. For any target $i$ and trajectory perturbation $\delta\in\mathbb R^L$,
\begin{equation}
\bigl[G^\star(Z+\delta e_i^\top)\bigr]_{:,i}=\bigl[G^\star(Z)\bigr]_{:,i},
\quad
E(Z+\delta e_i^\top)_{:,i}-E(Z)_{:,i}=\delta .
\label{eq:target-exclusion}
\end{equation}
\end{theorem}
\begin{proof}
By Eq.~\ref{eq:neural-closure}, every message entering target $i$ is a function of $z_j$ with $j\ne i$ only, and the frozen coefficients and deterministic gates are input-independent. Perturbing $z_i$ thus leaves $[G^\star(Z)]_{:,i}$ unchanged; subtracting it gives the residual identity.
\end{proof}

\noindent
No linearity is assumed: $\psi_\phi$ is nonlinear, but no source of target $i$ is $z_i$. A target-only change therefore enters the signed residual undiminished, and cannot be absorbed by copying the target. Absolute error need not grow, since a perturbation may cancel an existing residual.

\noindent\textbf{Changes explained by peers.}
For a general perturbation $\Delta$, let $U(Z,\Delta)=\Delta-[G^\star(Z+\Delta)-G^\star(Z)]$ be the part of the change the predictor fails to reproduce, with mean absolute magnitudes $\mathcal R(Z)=\norm{E(Z)}_{1,1}/(LD)$ and $\mathcal D(Z,\Delta)=\norm{U(Z,\Delta)}_{1,1}/(LD)$. Since $E(Z+\Delta)=E(Z)+U(Z,\Delta)$ requires no linearity, the reverse triangle inequality in both orders gives:
\begin{equation}
\bigl|\mathcal R(Z+\Delta)-\mathcal D(Z,\Delta)\bigr|\le\mathcal R(Z),
\quad
\bigl|\mathcal R(Z+\Delta)-\mathcal R(Z)\bigr|\le\mathcal D(Z,\Delta).
\label{eq:closure-geometry}
\end{equation}
When normal closure is small, the first bound places $\mathcal R(Z+\Delta)$ near the unexplained part of the change, and the two together give $\mathcal D(Z,\Delta)>2\mathcal R(Z)\Rightarrow\mathcal R(Z+\Delta)>\mathcal R(Z)$. Conversely, a coordinated change the peers reproduce keeps $\mathcal D$ small and the residual near its original value, however large $\norm{\Delta}$. Neither direction requires the reconstruction error to move, which is what makes the branches non-redundant.


\section{Experiments}

\subsection{Experimental Settings}

\noindent \textbf{Benchmarks.} 
We evaluate \method{} on five real-world multivariate time series datasets.
(\textit{\romannumeral1}) \textbf{MSL} is public data released from NASA~\citep{hundman2018detecting}, which contain the telemetry anomaly data derived from the Incident Surprise Anomaly reports of spacecraft monitoring systems.
(\textit{\romannumeral2}) \textbf{SWaT}~\citep{mathur2016swat} originates from a critical infrastructure system under 11 days of continuous operations with 51 sensors.
(\textit{\romannumeral3}) \textbf{PSM}~\citep{abdulaal2021practical} is collected internally from diverse application server nodes from eBay.
(\textit{\romannumeral4}) \textbf{SMD}~\citep{su2019robust} is a five-week-long dataset released from a large Internet company.
(\textit{\romannumeral5}) \textbf{SMAP}~\cite{hundman2018detecting} contains spacecraft telemetry and anomaly annotations from NASA's Soil Moisture Active Passive mission.
We additionally consider two NeurIPS-TS scenarios used by PhysDiff~\citep{physdiff2025}:
(\textit{\romannumeral6}) \textbf{GECCO}, named Water in the local release, records nine sensor variables from a water-monitoring system, including pH, temperature, and conductivity.
(\textit{\romannumeral7}) \textbf{Creditcard} contains 29 transaction features with rare fraudulent transactions marked as anomalies, extending the benchmark scope to financial fraud detection.
For fairness, we split all datasets into training $\Ntr$, validation $\Nval$, and calibration $\Ncal$ in a $70\%/15\%/15\%$ ratio.

\noindent \textbf{Baselines.} 
We conduct a comparative analysis of \method{} against 16 baselines for comprehensive evaluations, including 
(\textit{\romannumeral1}) \textit{the classic methods}: LOF~\citep{breunig2000lof}, iForest~\citep{liu2012isolation}; 
(\textit{\romannumeral2}) \textit{prediction-based methods}: CL-MPPCA~\citep{tariq2019detecting}, TimesNet~\citep{wu2022timesnet}, MTAD-GAT~\citep{zhao2020multivariate};
(\textit{\romannumeral3}) \textit{reconstruction-based methods}: LSTM-VAE~\citep{park2018multimodal}, BeatGAN~\citep{zhou2019beatgan}, InterFusion~\citep{li2021multivariate}, OmniAnomaly~\citep{su2019robust}.
Then, from the pespective of model architecture, we compare our model with: 
(\textit{\romannumeral4}) \textit{Transformer-based methods}: Anomaly Transformer (AnoTran)~\citep{xu2021anomaly}, TranAD~\citep{tuli2022tranad}, NPSR~\citep{lai2024nominality}, DCdetector~\citep{yang2023dcdetector}; 
and (\textit{\romannumeral5}) \textit{Diffusion-based methods}: DiffAD~\citep{xiao2023imputation}, ImDiffusion~\citep{chen2023imdiffusion}, D3R~\citep{wang2023drift}.

\noindent\textbf{Evaluation Metrics.}
To assess the detection performance of our proposed \method and the baseline models, we employ three standard metrics: affiliation precision, recall, and F1~\citep{huet2022local}, denoted Aff-P, Aff-R, and Aff-F1. These metrics measure the temporal proximity between predicted anomaly intervals and ground-truth events, with Aff-F1 defined as the harmonic mean of Aff-P and Aff-R. All values are reported as percentages, and no point adjustment is applied. To ensure statistical significance, we report the average results over five independent runs.

\begin{table*}[!t]
    \centering
    \caption{Anomaly detection performance on five benchmarks. P, R, and F1 denote affiliation metrics (\%). F1 is computed as the harmonic mean of the reported P and R. Bold and underlined F1 scores indicate the best and second-best performance.}
    \vspace{-10pt}
    \label{tab:physdiff-table1}
    \begingroup
    \footnotesize
    \setlength{\tabcolsep}{3.6pt}
    \renewcommand{\arraystretch}{1.05}
    \resizebox{\textwidth}{!}{%
    \begin{tabular}{c|ccc|ccc|ccc|ccc|ccc}
        \toprule
        \textbf{Dataset} & \multicolumn{3}{c|}{\textbf{SMD}} & \multicolumn{3}{c|}{\textbf{MSL}} & \multicolumn{3}{c|}{\textbf{SMAP}} & \multicolumn{3}{c|}{\textbf{SWaT}} & \multicolumn{3}{c}{\textbf{PSM}} \\
        \midrule
        \textbf{Metric} & P & R & F1 & P & R & F1 & P & R & F1 & P & R & F1 & P & R & F1 \\
        \midrule
        LOF          & 57.69 & 99.10 & 72.93 & 49.89 & 72.18 & 59.00 & 47.92 & 82.86 & 60.72 & 53.20 & 96.73 & 68.65 & 53.90 & 99.91 & 70.02 \\
        iForest      & 100.00 & 9.37 & 17.13 & 53.87 & 94.58 & 68.64 & 41.12 & 68.91 & 51.51 & 53.03 & 62.80 & 57.50 & 100.00 & 3.35 & 6.48 \\
        CL-MPPCA     & 58.29 & 88.43 & 70.26 & 50.10 & 76.30 & 60.48 & 48.50 & 80.20 & 60.45 & 54.32 & 90.16 & 67.79 & 55.27 & 93.46 & 69.46 \\
        TimesNet     & 58.50 & 82.32 & 68.40 & 52.10 & 88.47 & 65.58 & 42.40 & 88.10 & 57.25 & 48.37 & 84.23 & 61.45 & 65.16 & 94.23 & 77.04 \\
        MTAD-GAT     & 85.90 & 67.69 & 75.72 & 54.96 & 94.93 & 69.62 & 39.05 & 93.99 & 55.18 & 65.90 & 77.51 & 71.24 & 79.90 & 60.14 & 68.63 \\
        LSTM-VAE     & 70.26 & 78.53 & 74.17 & 60.30 & 91.26 & \underline{72.62} & 55.10 & 88.60 & 67.95 & 62.44 & 80.18 & 70.21 & 57.23 & 92.40 & 70.68 \\
        BeatGAN      & 74.11 & 81.64 & 77.69 & 55.74 & 98.94 & 71.31 & 54.04 & 98.30 & 69.74 & 61.89 & 83.46 & 71.07 & 58.81 & 99.08 & 73.81 \\
        InterFusion  & 56.20 & 90.10 & 69.22 & 48.40 & 92.30 & 63.50 & 42.10 & 93.21 & 58.00 & 50.20 & 88.50 & 64.06 & 60.30 & 95.60 & 73.95 \\
        DeepSVDD     & 64.98 & 64.77 & 64.87 & 10.53 & 100.00 & 19.05 & 29.73 & 7.09 & 11.45 & 59.11 & 93.53 & \underline{72.44} & 74.05 & 50.64 & 60.15 \\
        AnoTran      & 100.00 & 3.19 & 6.18 & 51.04 & 95.36 & 66.49 & 56.91 & 96.69 & \underline{71.65} & 53.63 & 59.94 & 56.61 & 52.01 & 82.18 & 63.70 \\
        TranAD       & 59.18 & 91.26 & 71.80 & 53.20 & 94.10 & 67.97 & 45.65 & 96.20 & 61.92 & 52.12 & 90.30 & 66.09 & 63.23 & 97.10 & 76.59 \\
        NPSR         & 57.45 & 89.82 & 70.08 & 50.50 & 92.61 & 65.36 & 44.20 & 95.10 & 60.35 & 49.67 & 88.20 & 63.55 & 61.18 & 96.30 & 74.82 \\
        DCdetector   & 50.93 & 95.57 & 66.45 & 55.94 & 95.53 & 70.56 & 53.12 & 98.37 & 68.99 & 53.25 & 98.12 & 69.03 & 54.72 & 86.36 & 66.99 \\
        DiffAD       & 60.25 & 92.50 & 72.97 & 55.10 & 95.30 & 69.83 & 48.30 & 97.10 & 64.51 & 54.20 & 91.56 & 68.09 & 62.85 & 98.20 & 76.65 \\
        ImDiffusion  & 61.36 & 93.25 & 74.02 & 56.22 & 96.11 & 70.94 & 49.17 & 97.80 & 65.44 & 55.45 & 92.31 & 69.28 & 63.27 & 98.60 & \underline{77.08} \\
        D3R          & 64.87 & 97.93 & \underline{78.04} & 56.45 & 95.55 & 70.97 & 51.08 & 94.46 & 66.31 & 64.25 & 77.50 & 70.26 & 53.17 & 100.00 & 69.43 \\
        \midrule
        \rowcolor{orbitblue} \method{}    & 81.65 & 95.73 & \textbf{88.13} & 61.53 & 96.85 & \textbf{75.25} & 58.90 & 98.41 & \textbf{73.69} & 63.14 & 86.71 & \textbf{73.07} & 65.28 & 98.76 & \textbf{78.60} \\
        \bottomrule
    \end{tabular}
    }
    \endgroup
\end{table*}

\begin{figure*}[!t]
    \centering
    \includegraphics[width=\textwidth]{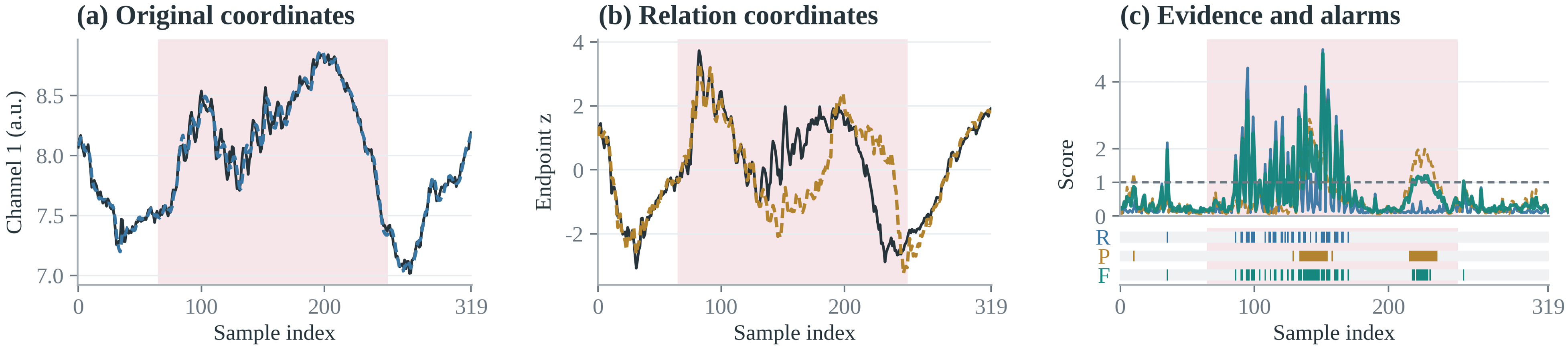}
    \vspace{-15pt}
    \caption{\textbf{Complementary scoring responses.} (a) Observation and reconstruction; (b) window-standardized observation and peer prediction; (c) scores and alarm strips. Black denotes observations; blue (R), reconstruction; gold (P), relation; and teal (F), equal-weight fusion. Pink shading marks the annotated anomalous interval. Scores are divided by their respective normal-calibrated thresholds, with the gray dashed line at one indicating the alarm boundary.}
    \label{fig:s}
\end{figure*}

\subsection{Main Results}
\label{sec:main-results}

\subsubsection{Quantitative Results}

\begin{table}[t]
\centering
\caption{Performance on the Creditcard and GECCO scenarios of NeurIPS-TS. P, R, and F1 denote affiliation metrics (\%). Bold and underlined F1 scores indicate the best and second-best.}
\label{tab:physdiff-table2}
\renewcommand{\arraystretch}{1.05}
\resizebox{\linewidth}{!}{
\begin{tabular}{ccccccc}
        \toprule
         & \multicolumn{3}{c|}{\textbf{Creditcard}} & \multicolumn{3}{c}{\textbf{GECCO}} \\
        \midrule
         & P & R & F1 & P & R & F1 \\
        \midrule
        iForest    & 1.57 & 88.16 & 3.09 & 2.55 & 97.35 & 4.97 \\
        BeatGAN    & 51.76 & 60.48 & 55.78 & 62.71 & 58.94 & \underline{60.77} \\
        AnoTran    & 55.86 & 57.55 & 56.69 & 54.77 & 57.21 & 55.96 \\
        DCdetector & 49.86 & 67.67 & 57.42 & 38.30 & 59.70 & 46.66 \\
        D3R        & 56.90 & 69.61 & \underline{62.62} & 59.00 & 36.61 & 45.18 \\
        \midrule
        \rowcolor{orbitblue} \method{} & 64.27 & 75.61 & \textbf{69.48} & 69.17 & 71.32 & \textbf{70.23} \\
        \bottomrule
    \end{tabular}
}
\vspace{-5pt}
\end{table}

\noindent\textbf{Performance across benchmarks.}
Table~\ref{tab:physdiff-table1} reports the highest Aff-F1 for \method{} among the listed methods on all five datasets. The largest margin is on SMD, where its $88.13$ exceeds D3R's $78.04$ by $10.09$ percentage points. On the spacecraft datasets, \method{} reaches $75.25$ on MSL and $73.69$ on SMAP, exceeding the strongest respective baselines, LSTM-VAE and AnoTran, by $2.63$ and $2.04$ points. The gain on PSM is $1.52$ points over ImDiffusion. SWaT presents a smaller gap: $73.07$ vs $72.44$ for DeepSVDD. The advantage in the reported F1 scores therefore spans the listed application domains, but its magnitude is dataset-dependent.

\noindent\textbf{Precision--recall balance.}
The individual metrics reveal different operating characteristics among the methods. On SMD, iForest and AnoTran both attain $100.00$ precision, but their recall falls below $10.00$. \method{} combines $81.65$ precision with $95.73$ recall, avoiding this pronounced imbalance. Relative to D3R, it increases precision by $16.78$ percentage points while reducing recall by $2.20$ points, resulting in a higher reported F1.

The improvements on MSL and SMAP involve increases in both metrics relative to the respective strongest F1 baselines. On MSL, \method{} improves precision and recall over LSTM-VAE by $1.23$ and $5.59$ percentage points. On SMAP, the corresponding gains over AnoTran are $1.99$ and $1.72$ points. On PSM, \method{} also improves both metrics over ImDiffusion, raising precision from $63.27$ to $65.28$ and recall from $98.60$ to $98.76$. By contrast, SWaT exhibits a precision--recall tradeoff: \method{} has higher precision than DeepSVDD ($63.14$ vs $59.11$), alongside lower recall ($86.71$ vs $93.53$).

\noindent\textbf{NeurIPS-TS results.}
Table~\ref{tab:physdiff-table2} extends the evaluation to the Creditcard and GECCO scenarios. \method{} achieves Aff-F1 scores of $69.48$ and $70.23$, exceeding the strongest listed baselines by $6.86$ and $9.46$ percentage points, respectively.

On Creditcard, \method{} improves over D3R in both precision ($64.27$ vs $56.90$) and recall ($75.61$ vs $69.61$). On GECCO, it similarly exceeds BeatGAN in precision ($69.17$ vs $62.71$) and recall ($71.32$ vs $58.94$), with a larger improvement in recall. In both scenarios, \method{} achieves the highest precision among the compared methods, while iForest attains higher recall at the cost of precision below $3.00$. The resulting F1 comparisons reinforce the importance of jointly considering precision and recall. They also show that the advantage of \method{} extends to both additional scenarios without requiring it to maximize recall alone.

\subsubsection{Qualitative Analysis}
\label{sec:qualitative-analysis}

Figure~\ref{fig:s} illustrates the scoring principle on a four-channel sequence. Panels (a) and (b) display the first channel in original and window-standardized coordinates, respectively, while all scores aggregate four channels. The reconstructor closely follows the observed trajectory, including the smooth deviation in the later part of the anomalous interval. Consequently, a visible departure in the trajectory does not necessarily produce a large reconstruction residual. Peer prediction departs from the observation over part of this interval, providing an additional criterion for identifying behavior that remains locally reconstructable.

The resulting responses are complementary in their temporal locations. Reconstruction scoring is most responsive to the earlier fluctuations, whereas relational scoring contributes additional alarms during the later departure. This distinction illustrates how agreement with a reconstructed trajectory and agreement with a relational reference can expose different aspects of the same anomalous interval.

Panel (c) shows how these responses combine after normal-data calibration. Reconstruction and relation scoring detect $39$ and $44$ anomalous points, respectively, with $10$ shared detections. Equal-weight fusion detects $58$ anomalous points, exceeding the individual branches by $19$ and $14$ points. The limited overlap between branch detections is consistent with the two criteria providing complementary information.

Fusion nevertheless operates on scores rather than on the union of binary alarms. The branch detections have a union of $73$ anomalous points, exceeding the $58$ detected by fusion, so some points detected by individual branches are not retained. Moreover, fusion is applied before normalization by its own calibrated threshold, and its displayed curve is therefore not the mean of the two normalized branch curves. Normal-point alarms also increase from $1$ for either branch to $2$ for fusion. The example thus illustrates how complementary responses can improve anomalous-point detection while retaining both missed detections and false alarms.

\begin{figure}[t]
    \centering
    \includegraphics[width=\linewidth]{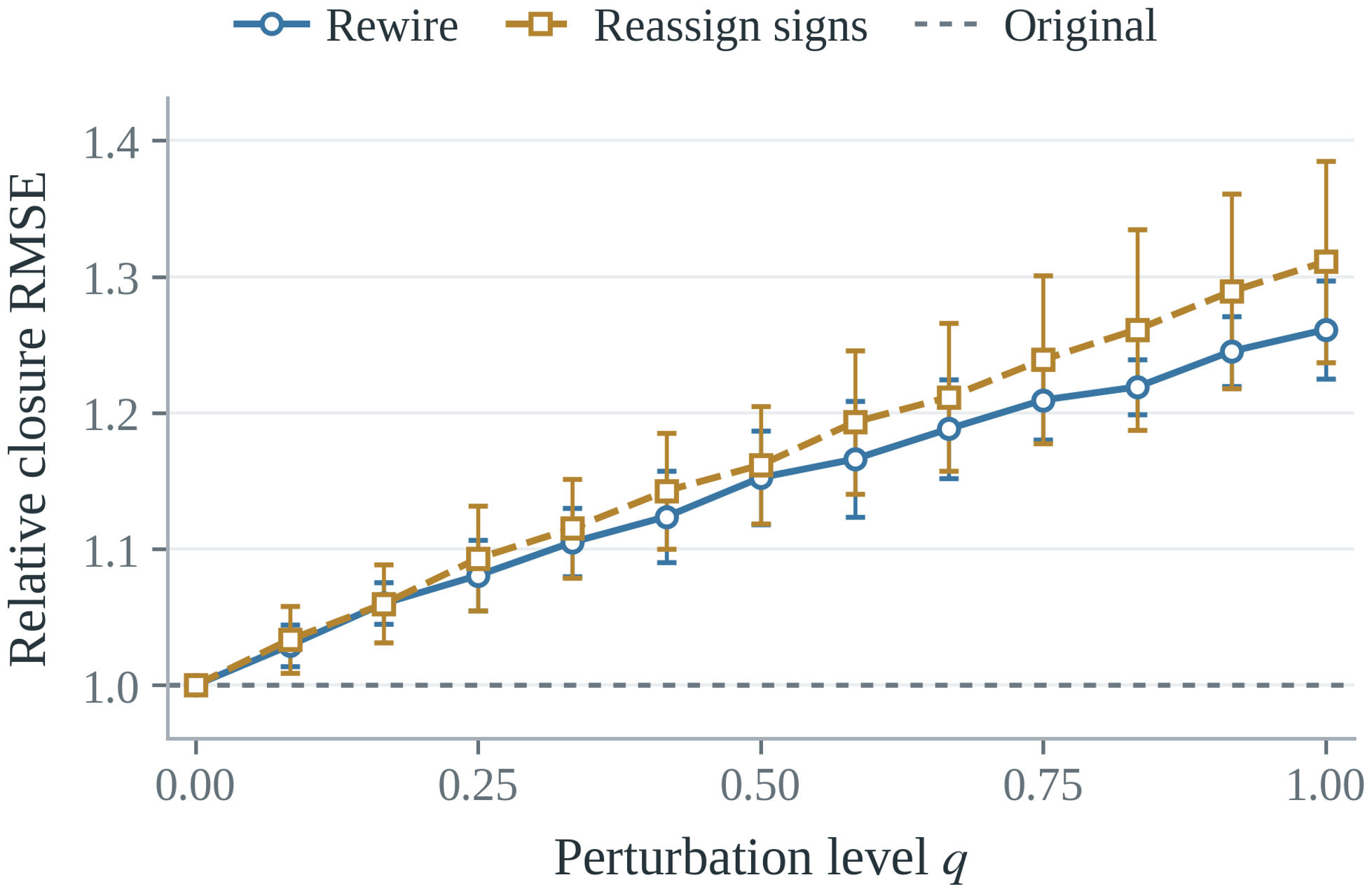}
    \vspace{-5pt}
    \caption{\textbf{Sensitivity to the relation reference.} Normalized closure RMSE (mean $\pm$ SD, $12$ paired trials). Blue: rewiring $q$ of the support; gold: flipping $q/2$ of edge signs; gray dashed: unperturbed baseline.}
    \vspace{-5pt}
    \label{fig:mechanism-reference}
\end{figure}

\subsection{Mechanism Analysis}
\label{sec:mechanism-analysis}

We hold observations fixed and perturb only their relational reference. Rewiring changes peer identities while preserving the incoming and outgoing degrees and each row's signed weight multiset. Sign reassignment preserves the connections and weight magnitudes while changing relation signs. Within each trial, both interventions use the same observations at every perturbation level, allowing us to examine the effect of reference changes on the residual.

Figure~\ref{fig:mechanism-reference} reports closure RMSE normalized by its unperturbed value within each trial on SMAP dataset. Across $12$ paired trials, the mean ratio reaches $1.261$ under rewiring and $1.311$ under sign reassignment at the largest perturbation level. These increases show that both peer identities and relation signs affect the residual, even when degrees or weight magnitudes are preserved. The two interventions are not budget-matched, so their ordering does not establish which aspect is more important. The results support the sensitivity of closure residuals to the relational reference, but do not by themselves establish relation-recovery accuracy or anomaly-detection performance.

\subsection{Fusion-Weight Sensitivity}
\label{sec:fusion-weight-sensitivity}

\begin{figure}[t]
    \centering
    \includegraphics[width=\linewidth]{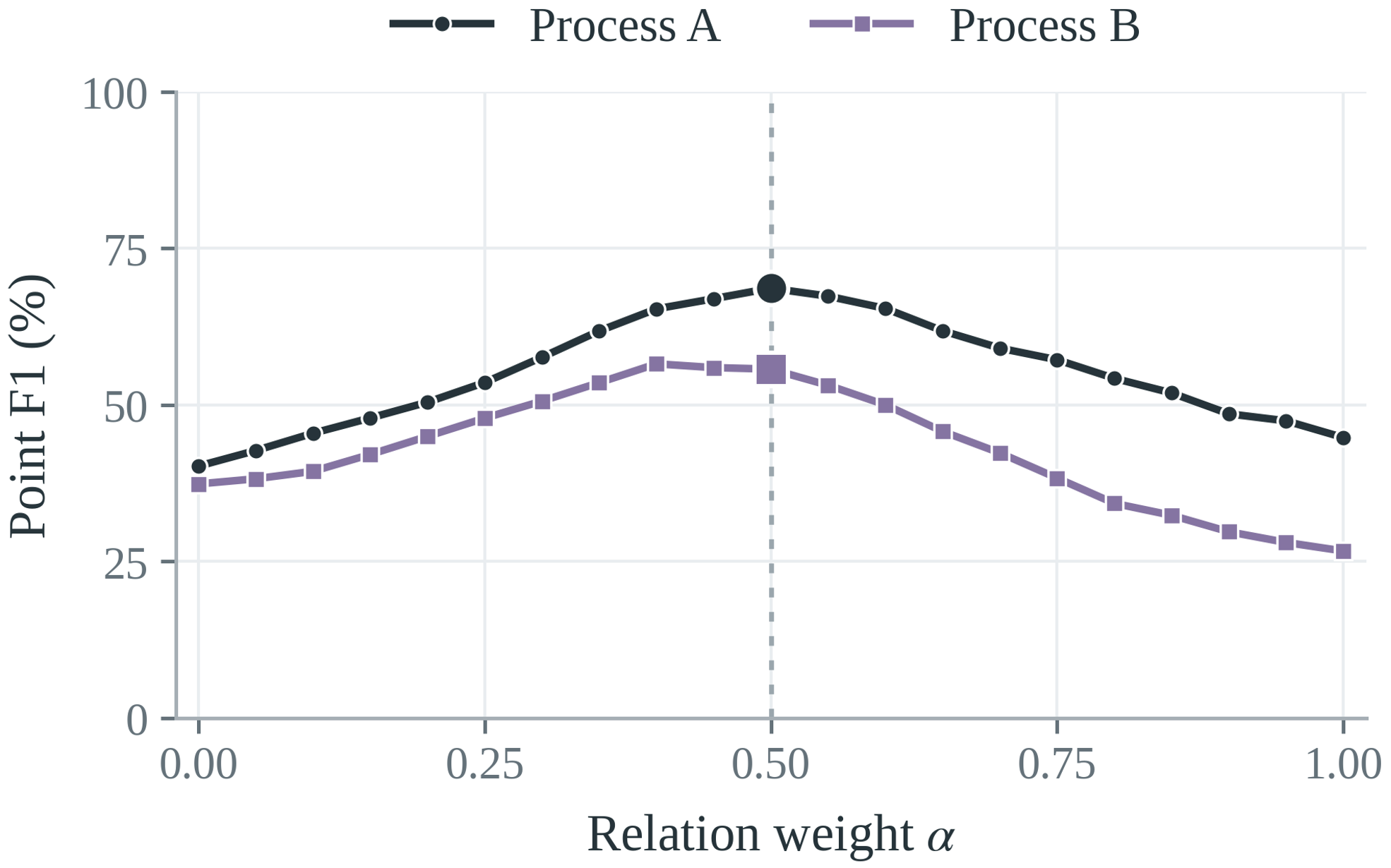}
    \vspace{-5pt}
    \caption{\textbf{Sensitivity to the fusion weight.} Unadjusted point F1 for Processes A and B, with a separately normal-calibrated $99$th-percentile threshold per weight. Endpoints represent individual branches; the dashed line and enlarged markers indicate equal-weight fusion.}
    \label{fig:fusion-weight-sensitivity}
    \vspace{-5pt}
\end{figure}

We examine how the balance between reconstruction and relation scoring affects detection performance. After branch standardization and softplus transformation calibrated on an independent normal sequence, we evaluate the fused score
$(1-\alpha)u^{\mathrm{rec}}+\alpha u^{\mathrm{str}}$
at $21$ weights. The detection threshold is recalibrated for each weight using the $99$th percentile of its normal scores.
Figure~\ref{fig:fusion-weight-sensitivity} shows that intermediate fusion weights outperform either standalone branch for both processes on SMD dataset. At $\alpha=0.5$, point F1 reaches $68.63\%$ and $55.71\%$ for Processes A and B, respectively, compared with $44.77\%$ and $37.37\%$ for their stronger standalone branches. This pattern is consistent with the two paths providing complementary evidence: intermediate weights retain contributions from both paths, whereas moving toward either endpoint reduces the contribution of one branch. The optimum is not necessarily symmetric. Process B peaks at $\alpha=0.4$, with $56.55\%$ F1, exceeding equal-weight fusion by $0.84$ percentage points. These results support equal weighting as an effective choice in the evaluated cases, while showing that the best balance can vary across processes.

\subsection{Event Coverage Analysis}
\label{sec:event-coverage-analysis}

\begin{figure}[t]
    \centering
    \includegraphics[width=\linewidth]{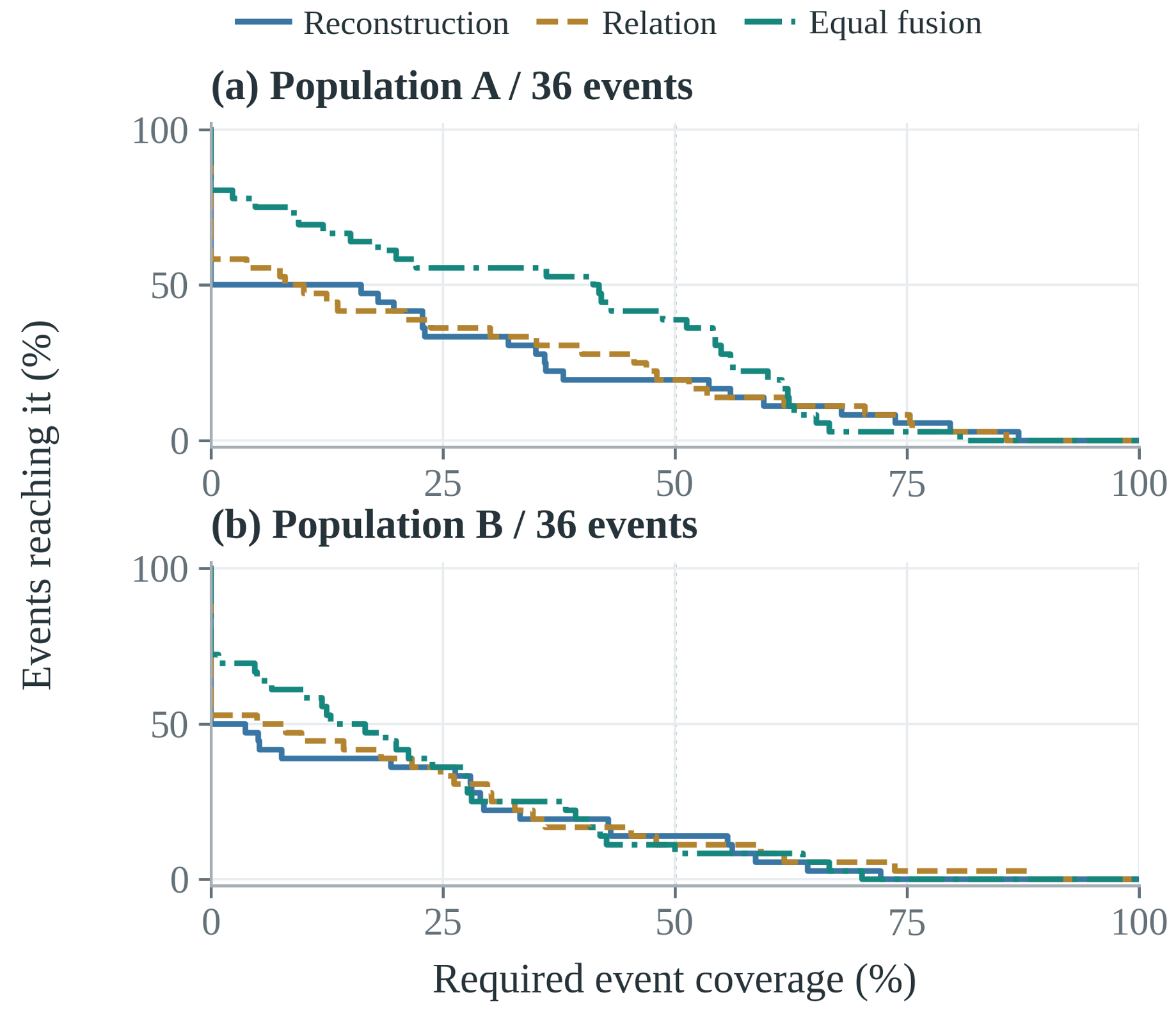}
    \vspace{-5pt}
    \caption{\textbf{Event retrieval and within-event coverage.} Percentage of $36$ events meeting each coverage requirement. Blue/gold/teal: reconstruction/relation/equal fusion.}
    \label{fig:event-coverage}
\end{figure}

We distinguish detecting an event from covering its anomalous points. Figure~\ref{fig:event-coverage} summarizes two event sets, labeled Populations A and B, each containing $36$ events on SMD dataset. Branch transformations and each rule's $99$th-percentile threshold are calibrated on an independent normal sequence. Event coverage is the fraction of anomalous points within an event receiving an alarm, without point adjustment or event completion. Each curve reports the percentage of events whose coverage is at least the required level, weighting every event equally. Event retrieval requires strictly positive coverage, whereas all events qualify at a required coverage of zero. This metric measures the extent of coverage rather than the temporal continuity of alarms.

In Population A, equal-weight fusion detects at least one point in $29$ events, compared with $18$ for reconstruction and $21$ for relation scoring. It covers at least half the anomalous points in $14$ events, versus $7$ for either branch. Population B shows a different tradeoff: fusion retrieves $26$ events, compared with $18$ for reconstruction and $19$ for relation scoring, yet reaches at least half coverage in only $4$ events, compared with reconstruction's $5$. Detecting more events therefore need not imply higher within-event coverage. Separately, the observed normal-point alarm rate in Population A is $1.55\%$ under fusion, compared with $0.74\%$ and $1.29\%$ for reconstruction and relation scoring, respectively. These results show that event retrieval, within-event coverage, and false alarms capture distinct aspects of detection performance; improvements in one do not necessarily imply improvements in the others.

\subsection{Ablation Study}

\begin{table}[t]
\centering
\caption{\textbf{Component ablations.} P, R, and F1 are affiliation metrics. The bold denotes the highest.}
\vspace{-5pt}
\label{tab:current-ablation}
\resizebox{\linewidth}{!}{
\begin{tabular}{ccccccc}
        \toprule
        & \multicolumn{3}{c}{\textbf{SWaT}} & \multicolumn{3}{c}{\textbf{SMD}} \\
        \cmidrule(lr){2-4}\cmidrule(l){5-7}
        \textbf{Variant} & P & R & F1 & P & R & F1 \\
        \midrule
        Rec.\ only & 60.42 & 83.15 & 69.99 & 72.60 & 84.50 & 78.10 \\
        Rel.\ only & 58.76 & 80.93 & 68.08 & 70.60 & 81.80 & 75.79 \\
        Raw bases & 61.85 & 84.92 & \underline{71.57} & 74.80 & 87.20 & \underline{80.53} \\
        \rowcolor{orbitblue} Full model & 63.14 & 86.71 & \textbf{73.07} & 81.65 & 95.73 & \textbf{88.13} \\
        \bottomrule
    \end{tabular}
}
\vspace{-5pt}
\end{table}


Table~\ref{tab:current-ablation} examines the two scoring paths and the learned basis memory. The single-path variants keep one path and omit the other; the raw-basis variant replaces the VQ codec with the original basis memory while retaining both scoring paths.

Adding relational scoring to reconstruction raises Aff-F1 from $69.99$ to $73.07$ on SWaT and from $78.10$ to $88.13$ on SMD. The full model also exceeds relation-only scoring by $4.99$ and $12.34$ percentage points, respectively. The relational path is the weaker standalone detector here, yet removing it lowers both precision and recall, consistent with peer consistency acting as a complement rather than a replacement.
Replacing VQ memory with raw bases lowers Aff-F1 by $1.50$ and $7.60$ percentage points on SWaT and SMD, respectively. Since this variant still supplies normal basis templates and retains both scoring paths, the gap suggests a benefit from learning the basis representation beyond simply providing normal temporal context. This comparison evaluates the encoder, codebook, and decoder jointly and does not isolate the contribution of quantization alone.
\section{Conclusion}



We revisited a premise of reconstruction-based anomaly detection, that difficulty of recovery tracks abnormality. A residual certifies that a window can be reproduced, and leaves undetermined whether its channels still explain one another.\method{} evaluates recoverability and cross-channel consistency through two separately learned prediction tasks on normal data. A diffusion reconstructor conditioned on a discrete memory of normal temporal bases assesses recoverability. A relation-closure operator assesses cross-channel consistency by predicting each channel from its peers. The target's trajectory is withheld, and per-edge gates select the participating peers. Comprehensive experiments demonstrate the effectiveness of our methods across multiple benchmarks.

\bibliographystyle{ACM-Reference-Format}
\bibliography{sample-base}





\newpage
\section*{Appendices}
\appendix


\section{Baseline Descriptions}
\label{app:baselines}

We summarize the methods appearing in the main comparison tables and experimental setup; the descriptions concern their modeling principles rather than a claim that their original evaluation protocols are identical.

\begin{itemize}[leftmargin=1.2em,itemsep=4pt,topsep=4pt]
\item \textbf{LOF}~\citep{breunig2000lof} compares an observation's local density with those of its nearest neighbors and assigns a larger outlier score to observations in relatively sparse neighborhoods.
\item \textbf{iForest}~\citep{liu2012isolation} isolates observations through random recursive partitions and scores points that require shorter isolation paths as more anomalous.
\item \textbf{CL-MPPCA}~\citep{tariq2019detecting} combines convolutional LSTM modeling with mixtures of probabilistic principal component analyzers to capture temporal behavior and multiple normal operating patterns.
\item \textbf{TimesNet}~\citep{wu2022timesnet} reshapes a sequence into period-dependent two-dimensional tensors and uses convolutional blocks to model within-period and across-period variation.
\item \textbf{MTAD-GAT}~\citep{zhao2020multivariate} combines feature-wise and temporal graph attention with forecasting and reconstruction objectives to model dependencies across variables and time.
\item \textbf{LSTM-VAE}~\citep{park2018multimodal} models normal sequences with recurrent latent-variable encoding and decoding, using the resulting probabilistic reconstruction to assess anomalous observations.
\item \textbf{BeatGAN}~\citep{zhou2019beatgan} regularizes time-series reconstruction through adversarial training and identifies anomalous regions from discrepancies between observations and generated reconstructions.
\item \textbf{InterFusion}~\citep{li2021multivariate} uses a hierarchical variational model to capture inter-metric and temporal dependencies and evaluates observations through their reconstruction probabilities.
\item \textbf{DeepSVDD}~\citep{zhou2021vae} learns a representation concentrated around a normal-data center and detects anomalies from their distance to that center rather than reconstruction error.
\item \textbf{Anomaly Transformer (AnoTran)}~\citep{xu2021anomaly} contrasts a temporal prior association with learned series associations and combines their discrepancy with reconstruction evidence for detection.
\item \textbf{TranAD}~\citep{tuli2022tranad} uses attention-based reconstruction with self-conditioning and adversarial training to model multivariate windows and score their reconstruction deviations.
\item \textbf{NPSR}~\citep{lai2024nominality} combines point-wise and sequence-wise reconstruction to estimate nominality and condition anomaly scores on the distinction between point and contextual deviations.
\item \textbf{DCdetector}~\citep{yang2023dcdetector} learns multiscale representations with an asymmetric dual-attention contrastive objective and derives detection evidence without a reconstruction objective.
\item \textbf{DiffAD}~\citep{xiao2023imputation} detects anomalies through conditional diffusion-based imputation, progressively adjusting conditioning weights to preserve information from observed values.
\item \textbf{ImDiffusion}~\citep{chen2023imdiffusion} learns diffusion-based time-series imputers and uses their denoised estimates to measure departures from temporal and cross-variable context.
\item \textbf{D3R}~\citep{wang2023drift} combines dynamic decomposition with diffusion reconstruction to separate nonstationary drift from the deviations used for anomaly detection.
\item \textbf{OmniAnomaly}~\citep{su2019robust}, also mentioned in the experimental setup, uses a stochastic recurrent model with a flexible variational posterior and reconstruction probabilities to detect multivariate anomalies.
\end{itemize}

\section{Datasets and Data Handling}
\label{app:datasets}

\subsection{Benchmark Descriptions}

\paragraph{Server Machine Dataset (SMD).}
SMD contains five weeks of monitoring data from 28 machines, with 38 recorded metrics per machine~\citep{su2019robust}. Its source release supplies separate training and test sequences and test anomaly annotations for each machine. Machine identities define separate sequences, not adjacent portions of one continuous trajectory.

\paragraph{Mars Science Laboratory (MSL).}
MSL contains spacecraft telemetry from the Curiosity rover and expert-annotated anomalous intervals~\citep{hundman2018detecting}. The source collection contains 27 telemetry streams. A stream's model inputs include a monitored telemetry variable and encoded command context; these input columns should not all be interpreted as distinct physical sensors.

\paragraph{Soil Moisture Active Passive (SMAP).}
SMAP is the companion NASA spacecraft collection with 55 telemetry streams and annotated anomalous intervals~\citep{hundman2018detecting}. As in MSL, the monitored signal is accompanied by command-related inputs. The number of streams in the collection is distinct from the feature dimension of an individual input window.

\paragraph{Secure Water Treatment (SWaT).}
SWaT records the operation of a six-stage water-treatment testbed through 51 sensor and actuator variables~\citep{mathur2016swat}. The commonly used release contains normal-operation data and a separate attack period. Measurements describe interacting processes such as flow, tank level, and valve operation, making cross-variable consistency relevant to monitoring.

\paragraph{Pooled Server Metrics (PSM).}
PSM contains multivariate monitoring signals collected from eBay application servers~\citep{abdulaal2021practical}. The released training and test files describe server behavior through anonymized metrics, with anomaly labels supplied separately for testing. It complements SMD with a different operational environment and dependency structure.

\paragraph{GECCO.}
The GECCO scenario in the NeurIPS time-series benchmark originates from drinking-water quality monitoring~\citep{lai2021revisiting}. It is named \texttt{Water} in the local release and contains nine measurement variables, including pH, temperature, and conductivity. Its annotations concern departures in a continuously monitored water-quality process.

\paragraph{Creditcard.}
The Creditcard scenario contains transaction records with fraud annotations and is included in the same benchmark~\citep{lai2021revisiting}. The local release uses 29 transaction features. Unlike regularly sampled sensor measurements, consecutive observations represent transaction events; window position therefore describes event order rather than a fixed physical time interval.

\section{Algorithmic Details and Pseudocode}
\label{app:algorithms}

We detail the independent fitting of the basis codec, reconstructor, and sparse relation operator in Section \ref{sec:method}, followed by calibration of their equal-weight score. The relational formulation assumes $D\ge2$ retained channels. We write $\mathcal W_q$ for the windows from segment $q\in\{\mathrm{tr},\mathrm{val},\mathrm{cal},\mathrm{test}\}$ and store each input window together with its aligned detrended copy.

\subsection{Fitting the Three Components}
\label{app:codec-objective}

\paragraph{Basis codec.}
For a normalized basis patch $b$, let $z=E_u(b)$, $c=c^u_{q_u(b)}$, and $\widehat b=D_u(z+\sg[c-z])$, as in Eq. \ref{eq:memory-quantization}. Each band minimizes the usual straight-through quantization objective~\citep{van2017neural}:
\begin{equation}
\mathcal L_{\mathrm{VQ}}^u
=\mathbb E_b\!\left[
\operatorname{MSE}(b,\widehat b)
+\operatorname{MSE}(\sg[z],c)
+\beta_{\mathrm{vq}}\operatorname{MSE}(z,\sg[c])
\right].
\label{eq:app-vq-loss}
\end{equation}
Here $\operatorname{MSE}$ averages squared coordinate differences, and $\beta_{\mathrm{vq}}>0$ weights commitment to the selected code. The codebook loss updates the selected code vector, while the straight-through reconstruction gradient and commitment loss update the encoder. The fitted codecs and template indices are subsequently frozen. Each stored template retains paired high- and low-frequency code sequences; the two bands are read with the same template index.

\paragraph{Reconstruction and relation fitting.}
The reconstructor minimizes $\mathcal L_{\mathrm{rec}}$ with the frozen memory. Independently, the relation predictor minimizes $\mathcal L_{\mathrm{rel}}$ over all retained targets. Its candidates comprise all $D(D-1)$ directed off-diagonal pairs, not just the nonzero entries of $C_0$. One hard-concrete gate vector is drawn per training minibatch; validation uses deterministic gates. The scored target set is not changed by pruning.

\section{Detailed Theoretical Analysis}
\label{app:proofs}

We prove Lemma~\ref{lem:scale-erasure}, Theorem~\ref{thm:target-exclusion}, and Eq. \ref{eq:closure-geometry}, and derive the gate penalty. The results concern a frozen operator in fixed standardized coordinates. They do not imply that every anomaly violates its relations, that the branches are statistically independent, or that fusion improves detection. Nor does an empirical normal-data quantile guarantee a distribution-free false-alarm rate under dependence or distribution shift. Detection accuracy and event coverage require empirical evaluation.

\subsection{Proof of the Normalization Lemma}
\label{app:normalization-proof}

Fix a row $i$, let $s=\sum_j|E_{ij}|>0$, and take $\alpha>0$ and $\epsilon>0$. Positive scaling preserves signs, so
\begin{equation}
\begin{aligned}
\mathcal W_\epsilon(\alpha E)_{ij}-\mathcal W_0(E)_{ij}
&=\frac{\alpha E_{ij}}{\alpha s+\epsilon}-\frac{E_{ij}}{s}\\
&=-\frac{\epsilon E_{ij}}{s(\alpha s+\epsilon)}.
\end{aligned}
\label{eq:app-normalization-difference}
\end{equation}
Taking absolute values and summing over $j$ proves the exact identity
\begin{equation}
\sum_j\left|\mathcal W_\epsilon(\alpha E)_{ij}-\mathcal W_0(E)_{ij}\right|
=\frac{\epsilon\sum_j|E_{ij}|}{s(\alpha s+\epsilon)}
=\frac{\epsilon}{\alpha s+\epsilon}.
\label{eq:app-normalization-proof}
\end{equation}
The right-hand side tends to zero when $\alpha s/\epsilon\to\infty$, establishing Lemma~\ref{lem:scale-erasure}. If $\epsilon=0$, every positive scaling gives exactly the same normalized row. The case $s=0$ is excluded from the lemma because $\mathcal W_0(E)_{i:}$ is undefined; the implemented $\epsilon$-regularized map instead returns a zero row.

\paragraph{Why the order of gating matters.}
Let $s_i=\sum_{j\ne i}|\omega_{ij}|$ and let $m_i$ be the absolute row mass after deterministic gating. Since $0\le\bar g_{ij}\le1$,
\begin{equation}
m_i=\sum_{j\ne i}|b_{ij}|\bar g_{ij}
=\frac{\sum_{j\ne i}|\omega_{ij}|\bar g_{ij}}{s_i+\epsilon}
\le\frac{s_i}{s_i+\epsilon}\le1.
\label{eq:app-gate-mass}
\end{equation}
If all gates in a row equal $c\in[0,1]$, its mass is $cs_i/(s_i+\epsilon)$: the gates retain control over attenuation after normalization. A second exact normalization would remove this control because, for $c>0$ and a nonzero row $b_{i:}$,
\begin{equation}
\frac{c b_{ij}}{\sum_k|c b_{ik}|}=\frac{b_{ij}}{\sum_k|b_{ik}|}.
\label{eq:app-renormalization}
\end{equation}
With a positive numerical stabilizer, cancellation is approximate when the row mass dominates that stabilizer, as quantified by the lemma. Normalization thus preserves relative contributions but not their absolute scale. This motivates separate attenuation parameters; it does not prove that learned sparsity necessarily improves detection.

\subsection{Expected Gate Count and Deterministic Evaluation}
\label{app:gate-derivation}

For the hard-concrete gate in Eq. \ref{eq:sparse-gates}, assume $T_g>0$ and $\gamma<0<1<\zeta$. The event $g_{ij}>0$ is unaffected by clipping at the upper endpoint. Writing $u=u_{ij}$ and $\alpha=\alpha_{ij}$ gives
\begin{equation}
\begin{aligned}
g_{ij}>0
&\iff \sigma\!\left(\frac{\log\frac{u}{1-u}+\alpha}{T_g}\right)
>\frac{-\gamma}{\zeta-\gamma}\\
&\iff \log\frac{u}{1-u}>T_g\log\frac{-\gamma}{\zeta}-\alpha.
\end{aligned}
\label{eq:app-gate-event}
\end{equation}
Because $u\sim\mathcal U(0,1)$, $\Pr(\log\frac{u}{1-u}\le x)=\sigma(x)$, and therefore
\begin{equation}
\pi_{ij}:=\Pr(g_{ij}>0)
=\sigma\!\left(\alpha_{ij}-T_g\log\frac{-\gamma}{\zeta}\right),
\mathcal S(\alpha)=\frac{1}{D(D-1)}\sum_{i\ne j}\pi_{ij}.
\label{eq:app-expected-gates}
\end{equation}
This is the expected fraction of open candidate gates, following the hard-concrete relaxation of \citet{louizos2017learning}; it is not a confidence probability that an edge is physically correct. An open gate can still multiply a zero coefficient, so gate counts and nonzero effective-coefficient counts need not agree.

Validation and detection replace sampling by the deterministic surrogate
\begin{equation}
\bar g_{ij}=\operatorname{clip}_{[0,1]}
\left(\gamma+(\zeta-\gamma)\sigma(\alpha_{ij})\right).
\label{eq:app-deterministic-gates}
\end{equation}
The reference implementation uses $T_g=2/3$, $\gamma=-0.1$, and $\zeta=1.1$. In general, $\bar g_{ij}\ne\mathbb E[g_{ij}]$. Closing every incoming gate of target $i$ sets its prediction to zero, but its Huber loss remains in the objective; pruning cannot remove the target from supervision or scoring.

\subsection{Proof of Target-Excluded Closure Response}
\label{app:target-proof}

Fix the fitted parameters, channel embeddings, and deterministic gates. Then target $i$ has the functional form
\begin{equation}
G_i^\star(Z)=\sum_{j\ne i}w_{ij}^\star\,
\psi_{\phi^\star}(z_j;v_i^\star,v_j^\star),
\qquad z_j=Z_{:,j}.
\label{eq:app-target-map}
\end{equation}
The weights and identity embeddings do not depend on the current observation. Let $Z'=Z+\delta e_i^\top$. For every $j\ne i$, $Z'_{:,j}=z_j$, hence every summand in Eq. \ref{eq:app-target-map} is unchanged and $G_i^\star(Z')=G_i^\star(Z)$. Subtracting predictions from observations yields
\begin{equation}
\begin{aligned}
E_i(Z')-E_i(Z)
&=(z_i+\delta-G_i^\star(Z'))-(z_i-G_i^\star(Z))\\
&=\delta.
\end{aligned}
\label{eq:app-target-proof}
\end{equation}
This proves both identities in Theorem~\ref{thm:target-exclusion}. The argument is exact for nonlinear source messages and arbitrary perturbation magnitude. At differentiable points, the equivalent Jacobian statement is
\begin{equation}
\frac{\partial G_i^\star}{\partial z_i}=0_{L\times L},
\qquad
\frac{\partial E_i}{\partial z_i}=I_L.
\label{eq:app-target-jacobian}
\end{equation}
The result requires exclusion from the entire forward path, not just a zero diagonal in one matrix. A target-dependent gate, a shared all-channel feature mixer, or multihop feedback through another channel could invalidate the argument. None is present in the specified one-hop operator. A change to channel $i$ can still alter predictions for other targets, since it is allowed to serve as their source.

\paragraph{Signed response versus error magnitude.}
The identity does not imply $\norm{E_i(Z')}_1>\norm{E_i(Z)}_1$. For example, choosing $\delta=-E_i(Z)$ makes the new target residual zero. What always holds is
\begin{equation}
\norm{\delta}_1-\norm{E_i(Z)}_1
\le\norm{E_i(Z')}_1
\le\norm{\delta}_1+\norm{E_i(Z)}_1.
\label{eq:app-target-norm}
\end{equation}
Thus $\norm{\delta}_1>2\norm{E_i(Z)}_1$ suffices to increase that target's absolute residual, but a conclusion about all-channel aggregation additionally requires accounting for the other targets.

\subsection{Proof of the General Perturbation Bounds}
\label{app:perturbation-proof}

For an arbitrary $\Delta\in\mathbb R^{L\times D}$, define
\begin{equation}
U=\Delta-\bigl(G^\star(Z+\Delta)-G^\star(Z)\bigr),
\qquad E'=E(Z+\Delta),\quad E=E(Z).
\label{eq:app-unexplained}
\end{equation}
Direct substitution gives the exact identity $E'=E+U$. Let $\norm{A}_{1,1}=\sum_{\ell,i}|A_{\ell i}|$. Applying the reverse triangle inequality first to $E'$ and $U$, then to $E'$ and $E$, gives
\begin{equation}
\begin{aligned}
\left|\norm{E'}_{1,1}-\norm{U}_{1,1}\right|
&\le\norm{E'-U}_{1,1}=\norm{E}_{1,1},\\
\left|\norm{E'}_{1,1}-\norm{E}_{1,1}\right|
&\le\norm{E'-E}_{1,1}=\norm{U}_{1,1}.
\end{aligned}
\label{eq:app-triangle}
\end{equation}
Division by $LD$ proves both inequalities in Eq. \ref{eq:closure-geometry}. In particular,
\begin{equation}
\mathcal D(Z,\Delta)-\mathcal R(Z)
\le\mathcal R(Z+\Delta)
\le\mathcal D(Z,\Delta)+\mathcal R(Z).
\label{eq:app-residual-bracket}
\end{equation}
If $\mathcal R(Z)\le\varepsilon_0$, the perturbed residual lies within $\varepsilon_0$ of the unexplained perturbation magnitude. If $\mathcal D(Z,\Delta)>2\mathcal R(Z)$, the lower bound is strictly larger than $\mathcal R(Z)$. Conversely, if $\mathcal D(Z,\Delta)\le\eta_0$, the second inequality in Eq. \ref{eq:closure-geometry} limits the change in residual magnitude to $\eta_0$, regardless of the size of $\Delta$ itself.

\paragraph{Linear special case.}
For $G^\star(Z)=ZW^\top$, the unexplained change becomes
\begin{equation}
U=\Delta(I-W^\top).
\label{eq:app-linear-case}
\end{equation}
Thus a perturbation satisfying $\Delta=\Delta W^\top$ leaves the residual unchanged. In the nonlinear case, the corresponding condition is the finite-increment equality $G^\star(Z+\Delta)-G^\star(Z)=\Delta$, not membership in a fixed linear subspace. These statements describe predictability under the learned operator, not a recovered causal graph or a universal definition of anomalous behavior.

\subsection{Coordinates, Endpoint Ratios, and Decision Scope}
\label{app:theory-scope}

\paragraph{Fixed standardized coordinates.}
The perturbations above act on $Z$, after preprocessing. If a raw window is perturbed and then standardized again, the relevant change is the \emph{induced} difference between the two standardized windows, not the raw additive perturbation. Channel-wise preprocessing preserves target exclusion from peer trajectories, but centering can remove an additive shift and rescaling can change its magnitude. Offline detrending also limits any claim about streaming detection delay.

\paragraph{The endpoint ratio is a different quantity.}
Suppress the window index and set $b=L^{-1}\sum_{\ell=1}^Lr_\ell$. For a second nonnegative residual profile $r'$, write $d_L=r'_L-r_L$ and $d_b=L^{-1}\sum_\ell(r'_\ell-r_\ell)$. The exact score difference is
\begin{equation}
\frac{r'_L}{b+d_b+\epsilon}-\frac{r_L}{b+\epsilon}
=\frac{d_L(b+\epsilon)-r_Ld_b}
{(b+\epsilon)(b+d_b+\epsilon)}.
\label{eq:app-ratio-change}
\end{equation}
Both denominators are positive. Therefore an endpoint-error increase raises the ratio only when $d_L(b+\epsilon)>r_Ld_b$; a simultaneous increase in the background can offset it. Since $b\ge r_L/L$,
\begin{equation}
0\le\frac{r_L}{b+\epsilon}<L,
\qquad
r_1=\cdots=r_L=c>0
\ \Longrightarrow\ 
\frac{r_L}{b+\epsilon}=\frac{c}{c+\epsilon}.
\label{eq:app-ratio-scope}
\end{equation}
A uniformly large residual therefore produces a ratio near one, not a score proportional to its absolute magnitude. This readout emphasizes local departures within a window and does not measure how long an anomaly persists.

\end{document}